\documentclass{article}

\usepackage{natbib}
\usepackage{hyperref}
\usepackage{url}
\usepackage{amsmath}
\usepackage{amssymb}
\usepackage{amsfonts}
\usepackage{amsthm}
\usepackage{mathtools}
\usepackage{algorithm}
\usepackage{algpseudocode}
\usepackage[margin=1.5in]{geometry}

\newtheorem{lemma}{Lemma}
\newtheorem{theorem}{Theorem}

\newtheorem{corollary}{Corollary}

\newtheorem*{rep@theorem}{\rep@title}
\newcommand{\newreptheorem}[2]{%
\newenvironment{rep#1}[1]{%
 \def\rep@title{#2 \ref{##1}}%
 \begin{rep@theorem}}%
 {\end{rep@theorem}}}
\newreptheorem{theorem}{Theorem}
\newreptheorem{lemma}{Lemma}
\newreptheorem{corollary}{Corollary}

\DeclareMathOperator*{\argmin}{\text{argmin}}

\DeclareMathOperator*{\diag}{\text{diag}}

\newcommand{\bR}{\mathbb{R}}

\newcommand{\cA}{\mathcal{A}}

\newcommand{\cK}{\mathcal{K}}

\newcommand{\cR}{\mathcal{R}}

\newcommand{\cO}{\mathcal{O}}

\newcommand{\comments}[1]{}
\newcommand{\ignore}[1]{}

\renewcommand{\phi}{\varphi}

\newcommand{\del}{\partial}

\newcommand{\bE}{\mathbb{E}}

\newcommand{\Reg}{\mathrm{Reg}}

\hypersetup{
  colorlinks   = true,
  urlcolor     = blue,
  linkcolor    = blue,
  citecolor    = blue
}

\title{Accelerating Optimization \\ via Adaptive Prediction}

\author{
Mehryar Mohri\\
Courant Institute and Google\\
251 Mercer Street\\
New York, NY 10012 \\
\texttt{mohri@cims.nyu.edu} \\
\and
Scott Yang \\
Courant Institute \\
251 Mercer Street \\
New York, NY 10012 \\
\texttt{yangs@cims.nyu.edu} \\
}

\begin{document}

\maketitle

\begin{abstract}
  We present a powerful general framework for designing data-dependent
  optimization algorithms, building upon and unifying recent
  techniques in adaptive regularization, optimistic gradient
  predictions, and problem-dependent randomization. We first present a
  series of new regret guarantees that hold at any time and under very
  minimal assumptions, and then show how different relaxations recover
  existing algorithms, both basic as well as more recent sophisticated
  ones. Finally, we show how combining adaptivity, optimism, and
  problem-dependent randomization can guide the design of
  algorithms that benefit from more favorable guarantees than recent
  state-of-the-art methods.
\end{abstract}

\section{Introduction}
\label{sec:intro}

Online convex optimization algorithms represent key tools in modern
machine learning.  These are flexible algorithms used for solving a
variety of optimization problems in classification, regression,
ranking and probabilistic inference.  These algorithms typically
process one sample at a time with an update per iteration that is
often computationally cheap and easy to implement. As a result, they
can be substantially more efficient both in time and space than
standard batch learning algorithms, which often have optimization costs
that are prohibitive for very large data sets.

In the standard scenario of online convex optimization
\citep{Zinkevich2003}, at each round $t = 1, 2, \ldots$, the learner
selects a point $x_t$ out of a compact convex set $\cK$ and incurs
loss $f_t(x_t)$, where $f_t$ is a convex function defined over $\cK$.
The learner's objective is to find an algorithm $\cA$ that minimizes
the regret with respect to a fixed point $x^*$:
$$\Reg_T(\cA, x^*) = \sum_{t = 1}^T f_t(x_t) - f_t(x^*)$$
that is the difference between the learner's cumulative loss and the
loss in hindsight incurred by $x^*$, or with respect to the loss of the
best $x^*$ in $\cK$,
$\Reg_T(\cA) = \max_{x^* \in \cK} \Reg_T(\cA, x^*)$.  We will assume
only that the learner has access to the gradient or an element of the
sub-gradient of the loss functions $f_t$, but that the loss functions
$f_t$ can be arbitrarily singular and flat, e.g.\ not necessarily
strongly convex or strongly smooth.  This is the most general setup of
convex optimization in the full information setting. It can be applied
to standard convex optimization and online learning tasks as well as
to many optimization problems in machine learning such as those of
SVMs, logistic regression, and ridge regression.  Favorable bounds in
online convex optimization can also be translated into strong learning
guarantees in the standard scenario of batch supervised learning using
online-to-batch conversion guarantees
\citep{Littlestone1989,CesaBianchiConconiGentile2004,MohriRostamizadehTalwalkar2012}.

In the scenario of online convex optimization just presented, minimax
optimal rates can be achieved by standard algorithms such as online
gradient descent \citep{Zinkevich2003}. However, general
minimax optimal rates may be too conservative.  Recently,
\emph{adaptive regularization} methods have been introduced for
standard descent methods to achieve tighter data-dependent
regret bounds (see \citep{BartlettHazanRakhlin2007},
\citep{DuchiHazanSinger2010}, \citep{McMahanStreeter2010},
\citep{McMahan2014}, \citep{OrabonaCrammerCesaBianchi2013}).
Specifically, in the ``AdaGrad'' framework of \citep{DuchiHazanSinger2010}, there exists a sequence of convex functions 
$\psi_t$ such that the update $x_{t+1} = \argmin_{x \in \cK} \eta g_t^\top x + B_{\psi_t}(x,x_t)$
yields regret:
$$\Reg_T(\cA, x) \leq \sqrt{2} \max_t\|x - x_t\|_\infty \sum_{i=1}^n \sqrt{\sum_{t=1}^T |g_{t,i}|^2},$$
where $g_t \in \del f_t(x_t)$ is an element of the subgradient of $f_t$ at
$x_t$, $g_{1:T, i} = \sum_{t=1}^T g_{t,i}$, and $B_{\psi_t}$ is the Bregman divergence defined using the convex function $\psi_t$. 
This upper bound on the regret has shown to be within a factor $\sqrt{2}$ of the
optimal a posteriori regret: 
$$\sqrt{ n \inf_{s\succcurlyeq 0, \langle 1, s \rangle \leq n} \sum_{t=1}^T \|g_t\|_{\diag(s)^{-1}}^2}.$$
Note, however, that this upper bound on the regret can still be very large,
even if the functions $f_t$ admit some favorable properties (e.g.\
$f_t \equiv f$, linear). This is because the dependence is directly on
the norm of $g_t$s.

An alternative line of research has been investigated by a series of
recent publications that have analyzed online learning in
``slowly-varying'' scenarios
\citep{HazanKale2009,ChiangYangLeeMahdaviLuJinZhu2012,RakhlinSridharan2013,ChiangLeeLu2013}. In
the framework of \citep{RakhlinSridharan2013}, if $\cR$ is a
self-concordant function,
$\|\cdot\|_{\nabla^2 \cR(x_t)}$ is the semi-norm
induced by its Hessian at the point $x_t$,\footnote{The norm induced by a symmetric positive
definite (SPD) matrix $A$ is defined for any $x$ by
$\|x\|_A = \sqrt{x^\top A x}$.} and
$\tilde{g}_{t + 1} = \tilde{g}_{t + 1}(g_1,\ldots, g_t, x_1,\ldots,
x_t)$
is a ``prediction'' of a time $t + 1$ subgradient
$g_{t + 1}$ based on information up to time $t$, then one can
obtain regret bounds of the following form:
$$\Reg_T(\cA, x) \leq \frac{1}{\eta} \cR(x) + 2 \eta \sum_{t = 1}^T
\|g_t - \tilde{g}_{t}\|_{\nabla^2 \cR(x_t),*} \, .$$
Here,  $\| \cdot \|_{\nabla^2 \cR(x_t), *}$ denotes the dual norm of
$\| \cdot \|_{\nabla^2 \cR(x_t)}$: for any $x$,
$\| x \|_{\nabla^2 \cR(x_t), *} = \sup_{\| y \|_{\nabla^2 \cR(x_t)} \leq 1} x^T y$.  This
guarantee can be very favorable in the \emph{optimistic} case where
$\tilde{g}_{t} \approx g_t$ for all $t$. Nevertheless, it admits
the drawback that much less control is available over the induced norm
since it is difficult to predict, for a given self-concordant function $\cR$,
the behavior of its Hessian at the points $x_t$ selected by an
algorithm. Moreover, there is no guarantee of ``near-optimality'' with respect to
an optimal a posteriori regularization as there is with the adaptive algorithm. 

This paper presents a powerful general framework for designing online
convex optimization algorithms combining adaptive regularization and
optimistic gradient prediction which helps address several of the
issues just pointed out.  Our framework builds upon and unifies recent
techniques in adaptive regularization, optimistic gradient
predictions, and problem-dependent randomization.  In
Section~\ref{sec:ao_ftrl}, we describe a series of \emph{adaptive and
  optimistic} algorithms for which we prove strong regret guarantees,
including a new \emph{Adaptive and Optimistic
  Follow-the-Regularized-Leader} (AO-FTRL) algorithm
(Section~\ref{sec:ao_ftrl_algo}) and a more general version of this
algorithm with composite terms (Section~\ref{sec:cao_ftrl_algo}).
These new regret guarantees hold at any time and under very minimal
assumptions. We also show how different relaxations recover both basic
existing algorithms as well as more recent sophisticated ones. In a specific application,
we will also show how a certain choice of regularization functions will produce an
optimistic regret bound that is also nearly a posteriori optimal, combining the two different desirable properties mentioned above. 
Lastly, in Section~\ref{sec:aos_ftrl}, we further combine adaptivity and
optimism with problem-dependent randomization to devise
algorithms benefitting from more favorable guarantees than recent
state-of-the-art methods.

\section{Adaptive and Optimistic Follow-the-Regularized-Leader algorithms}
\label{sec:ao_ftrl}

\subsection{AO-FTRL algorithm}
\label{sec:ao_ftrl_algo}

In view of the discussion in the previous section, we present an adaptive and optimistic version 
of the Follow-the-Regularized-Leader (FTRL) family of algorithms. In each round of standard FTRL, a point 
is chosen that is the minimizer of the average linearized loss incurred plus a regularization term. In our new version of FTRL,  
we will find a minimizer of not only the average loss incurred, but also a prediction of the next round's loss. In addition, we 
will define a dynamic time-varying sequence of regularization functions that can be used to optimize against this new loss term.    
Algorithm~\ref{alg:ao_ftrl} shows the pseudocode of our Adaptive and
Optimistic Follow-the-Regularized-Leader (AO-FTRL) algorithm. 

\begin{algorithm}[t]
\caption{AO-FTRL}
\label{alg:ao_ftrl}
\begin{algorithmic}[1]
\State \textbf{Input:} regularization function $r_0 \geq 0$.
\State \textbf{Initialize: } $\tilde{g}_1 = 0$, $x_1 = \argmin_{x \in \cK} r_0(x)$.
\For{$t = 1, \ldots, T$:}
    \State Compute $g_t \in \del f_t(x_t)$. 
    \State Construct regularizer $r_t \geq 0$.
    \State Predict~gradient~$\tilde{g}_{t + 1}\!~=~\!\tilde{g}_{t + 1}(g_1,\ldots, g_t, x_1, \ldots, x_t)$. 
    \State Update $x_{t + 1} = \displaystyle \argmin_{x \in \cK} g_{1:t} \cdot x + \tilde{g}_{t + 1} \cdot x + r_{0:t}(x)$.
\EndFor
\end{algorithmic}
\end{algorithm}

The following result provides a regret guarantee for the algorithm when one uses proximal regularizers, i.e. \ functions $r_t$ such that
$\argmin_{x \in \cK} r_t(x) = x_t$.

\begin{theorem} [AO-FTRL-Prox]
\label{th:ao_ftrl-prox}
Let $\{r_t\}$ be a sequence of proximal non-negative functions, and let $\tilde{g}_t$ be the
learner's estimate of $g_t$ given the history of functions
$f_1,\ldots, f_{t-1}$ and points $x_1, \ldots, x_{t-1}$. Assume
further that the function
$h_{0:t}\colon x \mapsto g_{1:t}\cdot x + \tilde{g}_{t + 1}\cdot x +
r_{0:t}(x)$
is 1-strongly convex with respect to some norm $\|\cdot\|_{(t)}$
(i.e. $r_{0:t}$ is 1-strongly convex with respect to
$\|\cdot\|_{(t)}$).
% The fact that the $r_t$ are proximal means that $\argmin_x h_{0:t}(x) = \argmin_x h_{0:t}(x) - r_{t + 1}(x)$.
Then, the following regret bound holds for AO-FTRL
(Algorithm~\ref{alg:ao_ftrl}):
\begin{align*}
&\Reg_T(\text{\sc AO-FTRL}, x) = \sum_{t = 1}^T f_t(x_t) - f_t(x) 
\leq r_{0:T}(x) + \sum_{t = 1}^T \|g_t - \tilde{g}_t\|_{(t),*}^2 \, . 
\end{align*}
\end{theorem}

\begin{proof}
Recall that $x_{t + 1} 
= \argmin_x (g_{1:t} + \tilde{g}_{t + 1}) \cdot x + r_{0:t}(x)
= \argmin_x h_{0:t}(x)$, and let
$y_t = \argmin_x x \cdot g_{1:t} + r_{0:t}(x).$
Then, by convexity, the following inequality holds:
\begin{align*}
\sum_{t = 1}^T f_t(x_t) - f_t(x) 
& \leq \sum_{t = 1}^T g_t \cdot(x_t - x) \\
& = \sum_{t = 1}^T (g_t - \tilde{g}_t) \cdot(x_t - y_{t})  
+ \tilde{g}_t \cdot (x_t - y_{t}) + g_t \cdot (y_{t} - x) .
\end{align*}
Now, we first prove by induction on $T$ that for all $x \in \cK$ the
following inequality holds:
$$\sum_{t = 1}^T \tilde{g}_t \cdot (x_t - y_{t}) + g_t \cdot y_{t}
\leq \sum_{t = 1}^T g_t \cdot x + r_{0:T}(x).$$
For $T = 1$, since $\tilde{g}_1 = 0$ and $r_1 \geq 0$, the inequality
follows by the definition of $y_1$.  Now, suppose the inequality holds
at iteration $T$. Then, we can write
\begin{align*}
\sum_{t = 1}^{T+1} \tilde{g}_t \cdot (x_t - y_{t}) +  g_t \cdot y_{t} 
& \quad = \left[ \sum_{t = 1}^{T} \tilde{g}_t \cdot (x_t - y_{t}) +  g_t \cdot y_{t} \right] \\
& \quad \quad \quad + \tilde{g}_{T+1} \cdot (x_{T+1} - y_{T+1}) +  g_{T+1} \cdot y_{T+1} \\
& \quad \leq \left[ \sum_{t = 1}^{T} g_t \cdot x_{T+1} + r_{0:T}(x_{T+1}) \right] \\
& \quad \quad \quad + \tilde{g}_{T+1} \cdot (x_{T+1} - y_{T+1}) + g_{T+1} \cdot y_{T+1} \\
& \qquad \text{ (by the induction hypothesis for $x = x_{T+1}$)} \\
& \quad \leq \left[ \left( g_{1:T} + \tilde{g}_{T+1} \right) \cdot x_{T+1} + r_{0:T+1}(x_{T+1}) \right] \\
& \quad \quad \quad + \tilde{g}_{T+1} \cdot ( - y_{T+1}) + g_{T+1} \cdot y_{T+1} \\
& \qquad \text{ (since $r_t \geq 0$, $\forall t$)} \\
& \quad \leq \left[ \left( g_{1:T} + \tilde{g}_{T+1} \right) \cdot y_{T+1} + r_{0:T+1}(y_{T+1}) \right] \\
& \quad \quad \quad + \tilde{g}_{T+1} \cdot ( - y_{T+1}) + g_{T+1} \cdot y_{T+1}  \\
& \qquad \text{ (by definition of $x_{T+1}$)} \\
& \quad \leq g_{1:T+1} \cdot y + r_{0:T+1}(y), \text{ for any $y$.}\\
& \qquad \text{ (by definition of $y_{T+1}$)}
\end{align*}
Thus, we have that $\sum_{t = 1}^T f_t(x_t) - f_t(x) \leq r_{0:T}(x) + \sum_{t = 1}^T (g_t - \tilde{g}_t) \cdot (x_t - y_{t}) $
and it suffices to bound $\sum_{t = 1}^T (g_t - \tilde{g}_t)^T(x_t - y_{t})$.
Notice that, by duality, one can immediately write $(g_t - \tilde{g}_t)^T(x_t - y_{t}) \leq \|g_t - \tilde{g}_t\|_{(t),*} \|x_t - y_{t}\|_{(t)}$.
To bound $\|x_t - y_{t}\|_{(t)}$ in terms of the gradient, recall first that since $r_t$ is proximal and $x_t = \argmin_x h_{0:t-1}$, 
\begin{align*}
&x_t = \argmin_x h_{0:t-1}(x) + r_t(x), \\ 
&y_{t} = \argmin_x h_{0:t-1}(x) + r_t(x) + (g_t - \tilde{g}_t) \cdot x.
\end{align*}
The fact that $r_{0:t}(x)$ is $1$-strongly convex with respect to the norm $\|\cdot\|_{(t)}$ implies that $h_{0:t-1} + r_t$ is as well. In particular, it is 
$1$-strongly convex at the points $x_t$ and $y_{t}$. But this then implies that the conjugate function is $1$-strongly smooth 
on the image of the gradient, including at $\nabla(h_{0:t-1} + r_t)(x_t) = 0$ 
and 
$\nabla(h_{0:t-1} + r_t)(y_{t}) = -(g_t - \tilde{g}_t)$ (see Lemma~\ref{lem:smooth_cvx_duality} in the appendix or \citep{Rockafellar1970} for a general reference),
 which means that $\|\nabla ((h_{0:t-1} + r_t)^*)(-(g_t - \tilde{g}_t)) - \nabla ((h_{0:t-1} + r_t)^*)(0)\|_{(t)} \leq \|g_t - \tilde{g}_t\|_{(t),*}.$

Since $\nabla ((h_{0:t-1} + r_t)^*)(-(g_t - \tilde{g}_t)) = y_{t}$ and $\nabla ((h_{0:t-1} + r_t)^*)(0) = x_t$, we have that 
$\|x_t - y_{t}\|_{(t)} \leq \|g_t - \tilde{g}_t\|_{(t),*}$.
\end{proof}
The regret bound just presented can be vastly superior to the adaptive
methods of \citep{DuchiHazanSinger2010}, \citep{McMahanStreeter2010},
and others.  For instance, one common choice of gradient prediction is
$\tilde{g}_{t+1} = g_t$, so that for slowly varying gradients (e.g. nearly ``flat'' functions),
$g_t - \tilde{g}_{t} \approx 0$, but $\| g_t \|_{(t)} = \| g \|_{(t)}$.
Moreover, for reasonable gradient predictions,
$\|\tilde{g}_{t+1}\|_{(t)} \approx \|g_t\|_{(t)}$ generally, so that in the
worst case, Algorithm~\ref{alg:ao_ftrl}'s regret will be at most a
factor of two more than standard methods. At the same time, the use of non
self-concordant regularization allows one to more explicitly control
the induced norm in the regret bound as well as provide more efficient
updates than those of
\citep{RakhlinSridharan2013}. Section~\ref{sec:aogd} presents an
upgraded version of online gradient descent as an example, where our choice of
regularization allows our algorithm to \emph{accelerate} as the gradient predictions become more accurate.  

Note that the assumption of strong convexity of $h_{0:t}$ is not a significant constraint, as any 
quadratic or entropic regularizer from the standard mirror descent algorithms will satisfy this property. 

Moreover, if the loss functions $\{f_t\}_{t=1}^\infty$ themselves are $1$-strongly convex,
then one can set $r_{0:t} \equiv 0$ and still get a favorable induced
norm $\|\cdot\|_{(t),*}^2 = \frac{1}{t}\|\cdot\|_2^2$.  If the
gradients and gradient predictions are uniformly bounded, this
recovers the worst-case $\log(T)$ regret bounds. At the same time, 
Algorithm~\ref{alg:ao_ftrl} would also still
retain the potentially highly favorable data-dependent and
optimistic regret bound.

Liang and Steinhardt (2014) \citep{SteinhardtLiang2014} also studied
adaptivity and optimism in online learning in the context of mirror descent-type algorithms. 
If, in the proof above, we
assume their condition:
$$r_{0:t + 1}^*(-\eta g_{1:t}) \leq r_{0:t}^*(-\eta (g_{1:t} - \tilde{g}_t)) - \eta x_t^T(g_t - \tilde{g}_t),$$
then we obtain the following regret bound:
$\sum_{t = 1}^T f_t(x_t) - \sum_{t = 1}^T f_t(x) \leq \frac{r_1^*(0) + r_{0:T+1}(x)}{\eta}.$
Our algorithm, however, is generally easier to use since it holds for
any sequence of regularization functions and does not require
checking for that condition.

In some cases, it may be preferable to use non-proximal adaptive
regularization. Since non-adaptive non-proximal FTRL corresponds to
dual averaging, this scenario arises, for instance, when one wishes to
use regularizers such as the negative entropy to derive algorithms
from the Exponentiated Gradient (EG) family (see
\citep{Shalev-Shwartz2012} for background). We thus present the following theorem
for this family of algorithms: Adaptive Optimistic Follow-the-Regularized-Leader - General version (AO-FTRL-Gen).

\begin{theorem} [AO-FTRL-Gen]
\label{th:ao_ftrl-gen}
Let $\{r_t\}$ be a sequence of non-negative functions, and let
$\tilde{g}_t$ be the learner's estimate of $g_t$ given the history of
functions $f_1,\ldots, f_{t-1}$ and points $x_1, \ldots, x_{t-1}$.
Assume further that the function
$h_{0:t}\colon x \mapsto g_{1:t}\cdot x + \tilde{g}_{t + 1}\cdot x +
r_{0:t}(x)$
is 1-strongly convex with respect to some norm $\|\cdot\|_{(t)}$
(i.e. $r_{0:t}$ is 1-strongly convex wrt $\|\cdot\|_{(t)}$).  Then,
the following regret bound holds for AO-FTRL
(Algorithm~\ref{alg:ao_ftrl}):
$$\sum_{t = 1}^T f_t(x_t) - f_t(x) \leq r_{0:T-1}(x) + \sum_{t = 1}^T \|g_t - \tilde{g}_t\|_{(t-1),*}^2$$
\end{theorem}
Due to spatial constraints, the proof of this theorem, as well as that of
all further results in the remainder of Section~\ref{sec:ao_ftrl}, are
presented in Appendix~\ref{app:ao_ftrl}.

% Notice that the regret bound for AO-FTRL-Gen suffers from an ``off-by-one'' 
% phenomenon. Specifically, we are only able to specify up to the time $t-1$
% regularization function for the time $t$ gradient. This issue is also 
% encountered in previous work on adaptive regularization, including 
% the dual averaging version of AdaGrad in \citep{DuchiHazanSinger2010}.   

As in the case of proximal regularization,
Algorithm~\ref{alg:ao_ftrl} applied to general regularizers still
admits the same benefits over the standard adaptive algorithms.
In particular, the above algorithm is an easy upgrade over any dual
averaging algorithm. Section~\ref{sec:aoeg} illustrates one such
example for the Exponentiated Gradient algorithm.

\begin{corollary} 
  With the following suitable choices of the parameters in
  Theorem~\ref{th:cao_ftrl-prox}, the following regret bounds
  can be recovered:
\begin{enumerate}

\item Adaptive FTRL-Prox of \citep{McMahan2014} (up to a constant
  factor of 2): $\tilde{g} \equiv 0$.
 \item Primal-Dual AdaGrad of \citep{DuchiHazanSinger2010}:
   $r_{0:t} = \psi_t$, $\tilde{g} \equiv 0$.  
 \item Optimistic FTRL of \citep{RakhlinSridharan2013}:
   $r_0 = \eta \cR$ where $\eta > 0$ and $\cR$
   a self-concordant function, $r_t = \psi_t = 0, \forall t \geq 1$.

%  \item AdaGrad of Duchi et al (2010): 
% $$\tilde{g} \equiv 0, \quad r_{0:t} =   \sum_{s=1}^t \frac{\sqrt{\sum_{a=1}^s\|g_a - \tilde{g}_a\|^2} - \sqrt{\sum_{a=1}^{s-1}\|g_a - \tilde{g}_a\|^2} }{2 R^2}\|x-x_s\|_2^2,$$ 
% where $\cK \subset B_R$ the ball of radius $R$ under the Euclidean norm. 
\end{enumerate}
\end{corollary}

\subsection{Applications}
\label{sec:applications}

\subsubsection{Adaptive and Optimistic Gradient Descent}
\label{sec:aogd}

\begin{corollary}[AO-GD]
\label{cor:aogd}
Let $\cK \subset \times_{i=1}^n [-R_i,R_i]$ be an
$n$-dimensional rectangle, and denote $\Delta_{s,i} = \sqrt{\sum_{a = 1}^s (g_{a, i} - \widetilde g_{a, i})^2}$. Set
\begin{align*}
&r_{0:t} = \sum_{i=1}^n \sum_{s=1}^t \tfrac{\Delta_{s,i} - \Delta_{s-1,i} }{2 R_i} (x_i - x_{s,i})^2. 
\end{align*}

Then, if we use the martingale-type gradient
prediction $\tilde{g}_{t + 1} = g_t$, the following regret bound holds:
$$\text{Reg}_T(\text{\sc AO-GD}, x) \leq 4 \sum_{i=1}^n R_i \sqrt{\sum_{t = 1}^T (g_{t,i} - g_{t-1,i})^2}.$$
Moreover, this regret bound is nearly equal to the optimal a posteriori regret bound:
\begin{align*}
&R_i \sum_{i=1}^n \sqrt{\sum_{t = 1}^T (g_{t,i} - g_{t-1,i})^2} \\
&= 
\max_i R_i \sqrt{ n \inf_{s \succcurlyeq 0, \langle s, 1 \rangle \leq n} \sum_{t=1}^T \|g_t - g_{t-1}\|_{\diag(s)^{-1}}^2}. 
\end{align*}
\end{corollary}

Notice that the regularization function is minimized when the gradient predictions become more accurate. Thus, if we
interpret our regularization as an implicit learning rate, our algorithm uses a larger learning rate and \emph{accelerates} as our gradient predictions become
more accurate. This is in stark contrast to other adaptive regularization methods, such as AdaGrad, where learning rates are inversely proportional 
to simply the norm of the gradient. 

Moreover, since the regularization function decomposes over the coordinates, this acceleration can occur on a per-coordinate basis. If our gradient predictions 
are more accurate in some coordinates than others, then our algorithm will be able to adapt accordingly. Under the simple martingale prediction scheme, this means 
that our algorithm will be able to adapt well when only certain coordinates of the gradient are slowly-varying, even if the entire gradient is not.  

In terms of computation, the
AO-GD update can be executed in time linear in the dimension (the same
as for standard gradient descent). Moreover, since the gradient
prediction is simply the last gradient received, the algorithm also
does not require much more storage than the standard gradient descent
algorithm. However, as we mentioned in the general case, the
regret bound here can be significantly more favorable than the
standard $\cO\left(\sqrt{T G \sum_{i=1}^n R_i^2}] \right)$ bound of
online gradient descent, or even its adaptive variants.

\subsubsection{Adaptive and Optimistic Exponentiated Gradient}
\label{sec:aoeg}

\begin{corollary}[AO-EG]
\label{cor:aoeg}
Let $\cK = \Delta_n$ be the $n$-dimensional simplex and
$\phi\colon x \mapsto \sum_{i = 1}^n x_i \log(x_i)$ the negative
entropy. Assume that $\|g_t\| \leq C$ for all $t$ and set
$$r_{0:t} = \sqrt{2 \frac{C + \sum_{s=1}^t \|g_s -
    \tilde{g}_s\|_\infty^2}{\log(n)}} (\phi + \log(n)).$$
Then, if we use the martingale-type gradient prediction
$\tilde{g}_{t + 1} = g_t$, the following regret bound holds:
\begin{align*}
\Reg_T(\text{\sc AO-EG}, x) 
&\leq 2 \sqrt{2\log(n) \left( C + \sum_{t = 1}^{T-1} \|g_t - g_{t-1}\|_\infty^2 \right) }. 
\end{align*}
\end{corollary}
The above algorithm admits the same advantages over predecessors as
the AO-GD algorithm.  Moreover, observe that this bound holds at any
time and does not require the tuning of any learning rate. Steinhardt
and Liang \citep{SteinhardtLiang2014} also introduce a similar
algorithm for EG, one that could actually be more favorable if the
optimal a posteriori learning rate is known in advance. 

\subsection{CAO-FTRL algorithm (Composite Adaptive Optimistic Follow-the-Regularized-Leader)}
\label{sec:cao_ftrl_algo}

In some cases, we may wish to impose some regularization on our
original optimization problem to ensure properties such as
generalization (e.g.\ $l_2$-norm in SVM) or sparsity (e.g.\ $l^1$-norm
in Lasso). This ``composite term'' can be treated directly by modifying the
regularization in our FTRL update. However, if we wish for the
regularization penalty to appear in the regret expression but do not
wish to linearize it (which could mitigate effects such as
sparsity), then some extra care needs to be taken. 

We modify Algorithm~\ref{alg:ao_ftrl}
to obtain Algorithm~\ref{alg:cao_ftrl}, and we provide accompanying
regret bounds for both proximal and general regularization functions. 
In each theorem, we give a pair of regret bounds, depending on whether
the learner considers the composite term as an additional part of the 
loss.

All proofs are provided in Appendix~\ref{app:ao_ftrl}. 

\begin{algorithm}[t]
\caption{CAO-FTRL}
\label{alg:cao_ftrl}
\begin{algorithmic}[1]
\State \textbf{Input:} regularization function $r_0 \geq 0$, composite functions $\{\psi_t\}_{t = 1}^\infty$ where $\psi_t \geq 0$.
\State \textbf{Initialize: } $\tilde{g}_1 = 0$, $x_1 = \argmin_{x \in \cK} r_0(x)$.
\For{$t = 1, \ldots, T$:}
    \State Compute $g_t \in \del f_t(x_t)$. 
    \State Construct regularizer $r_t \geq 0$.
    \State Predict the next gradient $\tilde{g}_{t + 1} = \tilde{g}_{t + 1}(g_1,\ldots, g_t, x_1, \ldots, x_t)$. 
    \State Update $x_{t + 1} = \argmin_{x \in \cK} g_{1:t} \cdot x + \tilde{g}_{t + 1} \cdot x + r_{0:t}(x) + \psi_{1:t + 1}(x)$.
\EndFor
\end{algorithmic}
\end{algorithm}

\begin{theorem} [CAO-FTRL-Prox]
\label{th:cao_ftrl-prox}
Let $\{r_t\}$ be a sequence of proximal non-negative functions, such
that $\argmin_{x \in \cK} r_t(x) = x_t$, and let $\tilde{g}_t$ be
the learner's estimate of $g_t$ given the history of functions
$f_1,\ldots, f_{t-1}$ and points $x_1, \ldots, x_{t-1}$.  Let
$\{\psi_t\}_{t = 1}^\infty$ be a sequence of non-negative convex
functions, such that $\psi_1(x_1) = 0$. Assume further that the
function $h_{0:t}: x \mapsto g_{1:t} \cdot x + \tilde{g}_{t+1} \cdot x + r_{0:t}(x) + \psi_{1:t+1}(x)$ 
is 1-strongly convex with respect to some norm $\|\cdot\|_{(t)}$.
Then the following regret bounds hold for CAO-FTRL (Algorithm~\ref{alg:cao_ftrl}):
\begin{align*}
& \sum_{t = 1}^T f_t(x_t) - f_t(x) 
\leq  \psi_{1:T-1}(x) + r_{0:T-1}(x) + \sum_{t = 1}^T \|g_t - \tilde{g}_t\|_{(t-1),*}^2\\
& \sum_{t = 1}^T \left[ f_t(x_t) + \psi_t(x_t) \right] - \left[ f_t(x) + \psi_t(x) \right] 
\leq  r_{0:T}(x) + \sum_{t = 1}^T \|g_t - \tilde{g}_t\|_{(t),*}^2 \, .
\end{align*}
\end{theorem}

Notice that if we don't consider the composite term as part of our loss, then
our regret bound resembles the form of AO-FTRL-Gen. This is in spite of the fact that we are 
using proximal adaptive regularization. 
On the other hand, if the composite term is part of our loss, then our regret bound
resembles the one using AO-FTRL-Prox. 

\begin{theorem} [CAO-FTRL-Gen]
\label{th:cao_ftrl-gen}
Let $\{r_t\}$ be a sequence of non-negative functions, and let
$\tilde{g}_t$ be the learner's estimate of $g_t$ given the history of
functions $f_1,\ldots, f_{t-1}$ and points $x_1, \ldots, x_{t-1}$.
Let $\{\psi_t\}_{t = 1}^\infty$ be a sequence of non-negative convex
functions such that $\psi_1(x_1) = 0$. Assume further that the
function $h_{0:t}: x \mapsto g_{1:t} \cdot x + \tilde{g}_{t+1} \cdot x + r_{0:t}(x) + \psi_{1:t+1}(x)$ 
is 1-strongly convex with respect to some norm $\|\cdot\|_{(t)}$. 
Then, the following regret bound holds for CAO-FTRL (Algorithm~\ref{alg:cao_ftrl}):
\begin{align*}
& \sum_{t = 1}^T f_t(x_t) - f_t(x)
\leq  \psi_{1:T-1}(x) + r_{0:T-1}(x) + \sum_{t = 1}^T \|g_t - \tilde{g}_t\|_{(t-1),*}^2 \\
& \sum_{t = 1}^T f_t(x_t) + \psi_t(x_t) - \left[ f_t(x) + \psi_t(x) \right] 
\leq  r_{0:T-1}(x) + \sum_{t = 1}^T \|g_t - \tilde{g}_t\|_{(t),*}^2 \, .
\end{align*}
\end{theorem}

\section{Adaptive Optimistic and Stochastic Follow-the-Regularized-Leader algorithms}
\label{sec:aos_ftrl}

\subsection{CAOS-FTRL algorithm (Composite Adaptive Optimistic Follow-the-Regularized-Leader)}

We now generalize the scenario to that of stochastic online convex
optimization, where, instead of exact subgradient elements $g_t$, we
receive only estimates. Specifically, we assume access to a sequence
of vectors of the form $\hat{g}_t$, where
$\bE[\hat{g}_t | g_1,\ldots, g_{t - 1}, x_1,\ldots, x_t] = g_t$.  This
extension is in fact well-documented in the literature (see
\citep{Shalev-Shwartz2012} for a reference), and the extension of our
adaptive and optimistic variant follows accordingly.  
For completeness, we provide the proofs of the following theorems in
Appendix~\ref{app:aos_ftrl}.

\begin{algorithm}[t]
\caption{CAOS-FTRL}
\label{alg:caos_ftrl}
\begin{algorithmic}[1]
\State \textbf{Input:} regularization function $r_0 \geq 0$, composite functions $\{\psi_t\}_{t = 1}^\infty$ where $\psi_t \geq 0$.
\State \textbf{Initialize: } $\tilde{g}_1 = 0$, $x_1 = \argmin_{x \in \cK} r_0(x)$.
\For{$t = 1, \ldots, T$:}
    \State Query $\hat{g}_t$ where $\bE[\hat{g}_t | x_1,\ldots, x_t, \hat{g}_1,\ldots, \hat{g}_{t-1}] = g_t \in \del f_t(x_t)$. 
    \State Construct regularizer $r_t \geq 0$.
    \State Predict next gradient $\tilde{g}_{t + 1} = \tilde{g}_{t + 1}(\hat{g}_1,\ldots, \hat{g}_t, x_1, \ldots, x_t)$. 
    \State Update $x_{t + 1} = \argmin_{x \in \cK} \hat{g}_{1:t} \cdot x + \tilde{g}_{t + 1} \cdot x + r_{0:t}(x) + \psi_{1:t + 1}(x)$.
\EndFor
\end{algorithmic}
\end{algorithm}

\begin{theorem}[CAOS-FTRL-Prox]
\label{th:caos_ftrl-prox}
Let $\{r_t\}$ be a sequence of proximal non-negative functions, such
that $\argmin_{x \in \cK} r_t(x) = x_t$, and let $\tilde{g}_t$ be the
learner's estimate of $\hat{g}_t$ given the history of noisy gradients
$\hat{g}_1,\ldots, \hat{g}_{t-1}$ and points $x_1, \ldots, x_{t-1}$.
Let $\{\psi_t\}_{t = 1}^\infty$ be a sequence of non-negative convex
functions, such that $\psi_1(x_1) = 0$. Assume further that the
function $h_{0:t}(x) = \hat{g}_{1:t}\cdot x + \tilde{g}_{t + 1}\cdot x + r_{0:t}(x) + \psi_{1:t + 1}(x)$
is 1-strongly convex with respect to some norm $\|\cdot\|_{(t)}$.
Then, the update $x_{t + 1} = \argmin_x h_{0:t}(x)$ of
Algorithm~\ref{alg:caos_ftrl} yields the following regret bounds:
\begin{align*}
&\bE\left[ \sum_{t = 1}^T f_t(x_t) - f_t(x) \right] 
\leq \bE \left[ \psi_{1:T-1}(x) + r_{0:T-1}(x) + \sum_{t = 1}^T \|\hat{g}_t - \tilde{g}_t\|_{(t-1),*}^2 \right] \\
&\bE\left[ \sum_{t = 1}^T f_t(x_t) + \psi_t(x_t) - f_t(x) - \alpha_t \psi_t(x) \right] 
\leq \bE \left[ r_{0:T}(x) + \sum_{t = 1}^T \|\hat{g}_t - \tilde{g}_t\|_{(t),*}^2 \right].
\end{align*}

\end{theorem}

\begin{theorem}[CAOS-FTRL-Gen]
\label{th:caos_ftrl-gen}
Let $\{r_t\}$ be a sequence of non-negative functions, and let
$\tilde{g}_t$ be the learner's estimate of $\hat{g}_t$ given the
history of noisy gradients $\hat{g}_1,\ldots, \hat{g}_{t-1}$ and
points $x_1, \ldots, x_{t-1}$.  Let $\{\psi_t\}_{t = 1}^\infty$ be a
sequence of non-negative convex functions, such that
$\psi_1(x_1) = 0$. Assume furthermore that the function
$h_{0:t}(x) = \hat{g}_{1:t}\cdot x + \tilde{g}_{t + 1}\cdot x + r_{0:t}(x) + \psi_{1:t + 1}(x)$
is 1-strongly convex with respect to some norm $\|\cdot\|_{(t)}$. Then, the update $x_{t + 1} = \argmin_x h_{0:t}(x)$ of
Algorithm~\ref{alg:caos_ftrl} yields the regret bounds:
\begin{align*}
&\bE\left[ \sum_{t = 1}^T f_t(x_t) - f_t(x) \right]
\leq \bE \left[ \psi_{1:T-1}(x) + r_{0:T-1}(x) + \sum_{t = 1}^T \|\hat{g}_t - \tilde{g}_t\|_{(t-1),*}^2 \right] \\
&\bE\left[ \sum_{t = 1}^T f_t(x_t) + \psi_t(x_t) - f_t(x) - \psi_t(x) \right] 
\leq \bE \left[ r_{0:T-1}(x) + \sum_{t = 1}^T \|\hat{g}_t - \tilde{g}_t\|_{(t-1),*}^2 \right]
\end{align*}

\end{theorem}

The algorithm above enjoys the same advantages over its non-adaptive
or non-optimistic predecessors. Moreover, the choice of the adaptive
regularizers $\{r_t\}_{t = 1}^\infty$ and gradient predictions
$\{\tilde{g}\}_{t=1}^\infty$ now also depend on the randomness of the
gradients received. While masked in the above regret bounds, this
interplay will come up explicitly in the following two examples, where
we, as the learner, impose randomness into the problem.

\subsection{Applications}
\label{sec:stochastic_applications}

\subsubsection{Randomized Coordinate Descent with Adaptive Probabilities}
\label{sec:rcd}

Randomized coordinate descent is a method that is often used for very
large-scale problems where it is impossible to compute and/or store
entire gradients at each step. It is also effective for directly
enforcing sparsity in a solution since the support of the final
point $x_t$ cannot be larger than the number of updates introduced.

The standard randomized coordinate descent update is to choose a
coordinate uniformly at random (see e.g.\
\citep{Shalev-ShwartzTewari2011}).  Nesterov (2012) \citep{Nesterov2012}
analyzed random coordinate descent in the context of loss functions
with higher regularity and showed that one can attain better
bounds by using non-uniform probabilities.

In the randomized coordinate descent framework, at each round $t$ we
specify a distribution $p_t$ over the $n$ coordinates and pick a
coordinate $i_t \in \{1, \ldots, n\}$ randomly according to this
distribution. From here, we then construct an unbiased estimate of an
element of the subgradient:
$\hat{g}_t = \frac{(g_t \cdot e_{i_t}) e_{i_t}}{p_{t,i_t}}$.
This technique is common in the online learning literature,
particularly in the
context of the multi-armed bandit problem (see e.g.\
\citep{CesaBianchiLugosi2006} for more information).

The following theorem can be derived by applying
Theorem~\ref{th:caos_ftrl-prox} to the gradient estimates
just constructed. We provide a proof in Appendix~\ref{app:rcd}.

\begin{theorem}[CAO-RCD]
\label{th:cao_rcd}
Assume $\cK \subset \times_{i = 1}^n [-R_i,R_i]$. Let $i_t$ be a
random variable sampled according to the distribution $p_t$, and let
$$\hat{g}_t = \frac{(g_t \cdot e_{i_t}) e_{i_t}}{p_{t,i_t}}, \quad \hat{\tilde{g}}_t = \frac{(\tilde{g}_t \cdot e_{i_t}) e_{i_t}}{p_{t,i_t}},$$
be the estimated gradient and estimated gradient prediction. Denote $\Delta_{s,i} = \sqrt{\sum_{a=1}^s (\hat{g}_{a,i} -\hat{\tilde{g}}_{a,i})^2}$, and let
\begin{align*}
r_{0:t} = \text{ }\sum_{i = 1}^n \sum_{s=1}^t \frac{\Delta_{s,i} - \Delta_{s-1,i} }{2 R_i} (x_i-x_{s,i})^2 
\end{align*}
be the adaptive regularization. Then, the regret of the algorithm can be
bounded by:
\begin{align*}
&\bE \left[ \sum_{t = 1}^T f_t(x_t) + \alpha_t \psi(x_t) - f_t(x) - \alpha_t \psi(x) \right]
\leq 4 \sum_{i = 1}^n R_i \sqrt{ \sum_{t = 1}^T \bE \left[ \frac{ (g_{t,i}
      - \tilde{g}_{t,i})^2}{p_{t,i}} \right] } 
\end{align*}
\end{theorem}
In general, we do not have access to an element of the subgradient $g_t$
before we sample according to $p_t$. However, if we assume that we
have some per-coordinate upper bound on an element of the subgradient
uniform in time, i.e. $|g_{t,j}| \leq L_{j}$
$\forall t \in \{1,\ldots, T\}, j \in \{1,\ldots, n\}$, then we can
use the fact that
$|g_{t,j} - \tilde{g}_{t,j}| \leq \max \{ L_{j} - \tilde{g}_{t,j},
\tilde{g}_{t,j} \}$
to motivate setting $\tilde{g}_{t,j} := \frac{L_{j}}{2}$ and
$p_{t,j} = \frac{(R_j L_j)^{2/3}}{\sum_{k=1}^n (R_k L_{k})^{2/3}}$ (by
computing the optimal distribution). This yields the following
regret bound.

\begin{corollary}[CAO-RCD-Lipschitz]
\label{cor:cao_rcd_lip}
Assume that at any time $t$ the following per-coordinate Lipschitz
bounds hold on the loss function:
$|g_{t,i}| \leq L_{i}, \quad \forall i \in \{1,\ldots, n\}$. Set
$p_{t,i} = \frac{(R_i L_{i})^{2/3}}{\sum_{j=1}^n (R_j L_{j})^{2/3}}$
as the probability distribution at time $t$, and set
$\tilde{g}_{t,i} = \frac{L_{i}}{2}$.
Then, the regret of the algorithm can be bounded as follows:
\begin{align*}
&\bE \left[ \sum_{t = 1}^T f_t(x_t) + \alpha_t \psi(x_t) - f_t(x) - \alpha_t \psi(x) \right]
\leq 2 \sqrt{T} \left( \sum_{i = 1}^n (R_i L_i)^{2/3} \right)^{3/2} .
\end{align*}
\end{corollary}
An application of H\"{o}lder's inequality will reveal that this bound
is strictly smaller than the $2RL \sqrt{nT}$ bound one would obtain
from randomized coordinate descent using the uniform distribution. Moreover,
the algorithm above still entertains the intermediate data-dependent bound 
of Theorem~\ref{th:cao_rcd}. 

Notice the similarity between the sampling distribution generated here
with the one suggested by \citep{Nesterov2012}.  However, Nesterov assumed
higher regularity in his algorithm (i.e. $f_t \in C^{1,1}$) and
generated his probabilities from there. In our setting, we only need
$f_t \in C^{0,1}$. It should be noted that \citep{AfkanpourGyorgySzepesvariBowling2013} also
proposed an importance-sampling based approach to random coordinate descent for the specific setting of multiple kernel learning.
In their setting, they propose updating the sampling distribution at each point in time instead of using uniform-in-time Lipschitz constants, which 
comes with a natural computational tradeoff. Moreover, the introduction of adaptive per-coordinate learning rates in our algorithm allows for tighter 
regret bounds in terms of the Lipschitz constants.   

We can also derive the analogous mini-batch update:
\begin{corollary}[CAO-RCD-Lipschitz-Mini-Batch]
\label{cor:cao_rcd_lip_batch}
Assume $\cK \subset \times_{i = 1}^n [-R_i,R_i]$. Let  $\cup_{j=1}^k \{\Pi_j\} = \{1,\ldots, n\}$ be a partition of the coordinates, and let 
$e_{\Pi_j} = \sum_{i \in \Pi_j} e_i$. 
Assume we had the following Lipschitz condition on the partition: $\|g_t \cdot e_{\Pi_j} \| \leq L_{j}$ $\forall j \in \{1,\ldots, k\}$.  

Define $S_i =  \sum_{j \in \Pi_i} R_j $. Set $p_{t,i} = \frac{(S_i L_{i})^{2/3}}{\sum_{j=1}^k (S_j L_{j})^{2/3}}$ as the 
probability distribution at time $t$, and set $\tilde{g}_{t,i} = \frac{L_{i}}{2}$.

Then the regret of the resulting algorithm is bounded by:
\begin{align*}
&\bE \left[ \sum_{t = 1}^T f_t(x_t) + \alpha_t \psi(x_t) - f_t(x) - \alpha_t \psi(x) \right]
\leq 2 \sqrt{T} \left( \sum_{i = 1}^k (S_i L_i)^{2/3} \right)^{3/2}    
\end{align*} 
\end{corollary}
While the expression is similar to the non-mini-batch version, the $L_i$ and $R_i$ terms now have different meaning. Specifically, $L_i$ is a bound on the 
2-norm of the components of the gradient in each batch, and $R_i$ is the 1-norm of the corresponding sides of the hypercube.

\subsubsection{Stochastic Regularized Empirical Risk Minimization}
\label{sec:reg_erm}

Many learning algorithms can be viewed as instances of regularized
empirical risk minimization (e.g.\ SVM, Logistic Regression, Lasso),
where the goal is to minimize an objective function of the following
form:
$$H(x) = \sum_{j=1}^m f_j(x) + \alpha \psi(x).$$
If we denote the first term by $F(x) = \sum_{j=1}^m f_j(x)$, then
we can view this objective in our CAOS-FTRL framework, where
$f_t \equiv F$ and $\psi_t \equiv \alpha \psi.$
In the same spirit as for non-uniform random coordinate descent, we
can estimate the gradient of $H$ at $x_t$ by sampling according to
some distribution $p_t$ and use importance weighting to generate an unbiased estimate: If
$g_t \in \del F(x_t)$ and $g_t^j \in \del f_j(x_t)$, then
$$g_t = \sum_{j=1}^m g_t^j \approx \frac{g_t^{j_t}}{p_{t,j_t}}. $$

This motivates the design of an algorithm similar to the one derived for
randomized coordinate descent.
% \begin{theorem}
% \label{th:aos_erm}
% % Assume $\cK \subset B_R$. Let $j_t$ be a random variable sampled according to $p_t$, and let
% % $\hat{g}_t = \frac{g_t^{j_t}}{p_{t,j_t}},\quad \tilde{g}_t = \hat{g}_{t-1},$
% % be the estimated gradient and estimated gradient prediction. Let
% % $r_{0:t} = \sum_{s=1}^t \frac{\sqrt{\sum_{a=1}^s\|\hat{g}_a - \hat{g}_{a-1}\|_2^2} - \sqrt{\sum_{a=1}^{s-1}\|\hat{g}_a - \hat{g}_{a-1}\|_2^2} }{2 R}\|x-x_s\|_2^2$
% % be the adaptive regularization. 
% % 
% % Then the regret of the resulting algorithm is bounded by:
% % \begin{align*}
% % &\bE \left[ \sum_{t = 1}^T f_t(x_t) + \alpha \psi(x_t) - f_t(x) - \alpha \psi(x) \right]\\
% % &\quad \leq 4 R \sqrt{ \sum_{t = 1}^T \sum_{i,j=1}^m \left[ \frac{\left\| p_{t-1,j} g_t^i - p_{t,i} g_{t-1}^j \right\|_2^2}{p_{t,i} p_{t-1,j}} \right] }  
% % % 4 \sqrt{ \sum_{t = 1}^T \bE \left[ \left\| \hat{g}_t - \hat{g}_{t-1} \right\|^2 \right] } 
% % \end{align*}
% \end{theorem}
Here we elect to use as gradient prediction the last gradient of the current function being sampled $f_j$.
However, we may run into the problem of never seeing a function before. A logical modification would be to separate optimization into epochs
and do a full batch update over all functions $f_j$ at the start of each epoch. This is similar to the technique used in the 
Stochastic Variance Reduced Gradient (SVRG) algorithm of \cite{JohnsonZhang2013}. However, we do not assume extra function regularity as they do in their paper, 
so the bounds are not comparable. The algorithm is presented in Algorithm~\ref{alg:caos_reg-erm-epoch} and comes with the following guarantee: 

\begin{algorithm}[t]
\caption{CAOS-Reg-ERM-Epoch}
\label{alg:caos_reg-erm-epoch}
\begin{algorithmic}[1]
\State \textbf{Input:} scaling constant $\alpha > 0$, composite term $\psi$, $r_0 = 0$. 
\State \textbf{Initialize: } initial point $x_1 \in \cK$, distribution $p_1$. 
\State Sample $j_1$ according to $p_{1}$, and set $t=1$. 
\For{$s = 1, \ldots, k$:}
    \State Compute $\bar{g}_s^j = \nabla f_j(x_1)$ $\forall j \in \{1,\ldots, m\}$.
    \For{$a = 1, \ldots, T/k$:} 
      \State If $T \mod k = 0$, compute $g^j = \nabla f_j(x_t)$ $\forall j$. 
      \State Set $\hat{g}_t = \frac{g_t^{j_t}}{p_{t,j_t}}$, and construct $r_t \geq 0$.
      \State Sample $j_{t+1} \sim p_{t+1}$ and set $\tilde{g}_{t + 1} = \frac{\bar{g}_s^{j_t}}{p_{t,j_t}}$. 
      \State Update $x_{t + 1} = \argmin_{x \in \cK} \hat{g}_{1:t} \cdot x + \tilde{g}_{t + 1} \cdot x + r_{0:t}(x) + (t+1)\alpha\psi(x)$ and $t = t+1$. 
    \EndFor
\EndFor
\end{algorithmic}
\end{algorithm}
 
\begin{corollary}
\label{cor:aos_erm-epoch}
Assume $\cK \subset \times_{i = 1}^n [-R_i,R_i]$. Denote $\Delta_{s,i} = \sqrt{\sum_{a=1}^s(\hat{g}_{a,i} - \tilde{g}_{a,i})^2}$, and let 
$r_{0:t} = \sum_{i = 1}^n \sum_{s=1}^t \frac{\Delta_{s,i} - \Delta_{s-1,i} }{2 R_i}(x_i-x_{s,i})^2$
be the adaptive regularization. 

Then the regret of Algorithm~\ref{alg:caos_reg-erm-epoch} is bounded by:
\begin{align*}
&\bE \left[ \sum_{t = 1}^T f_t(x_t) + \alpha \psi(x_t) - f_t(x) - \alpha \psi(x) \right]
\leq \sum_{i=1}^n 4 R_i \sqrt{ \sum_{s=1}^k \sum_{t=(s-1)(T/k) + 1}^{(s-1)(T/k)+T/k} \sum_{j=1}^m \frac{\left|g_{t,i}^j - \bar{g}_{s,i}^j \right|^2}{p_{t,j} } }  
% 4 \sqrt{ \sum_{t = 1}^T \bE \left[ \left\| \hat{g}_t - \hat{g}_{t-1} \right\|^2 \right] } 
\end{align*}

Moreover, if $\|\nabla f_j\|_{\infty} \leq L_j$ $\forall j$, then setting $p_{t,j} = \frac{L_i}{\sum_{j=1}^m L_j}$ yields a worst-case bound of:
$8\sum_{i=1}^n R_i \sqrt{T \left(\sum_{j=1}^m L_j\right)^2}.$ 
\end{corollary}
 
We also include a mini-batch version of this algorithm in Appendix~\ref{app:reg_erm}, which 
can be useful due to the variance reduction of the gradient prediction. 

\section{Conclusion}
\label{sec:conclusion}

We presented a general framework for developing efficient adaptive and
optimistic algorithms for online convex optimization. Building upon
recent advances in adaptive regularization and predictable online
learning, we improved upon each method. We demonstrated the power of
this approach by deriving algorithms with better guarantees than those
commonly used in practice. In addition, we also extended adaptive and
optimistic online learning to the randomized setting. Here, we
highlighted an additional source of problem-dependent adaptivity (that
of prescribing the sampling distribution), and we showed how one can
perform better than traditional naive uniform sampling.

\newpage

\bibliographystyle{chicago}
% \bibliography{aos}

\newpage

\section*{Appendix}
\label{sec:appendix}

\section{Proofs for Section~\ref{sec:ao_ftrl}}
\label{app:ao_ftrl}

\begin{lemma}[Duality Between Smoothness and Convexity for Convex Functions]
\label{lem:smooth_cvx_duality}
Let $\cK$ be a convex set and $f:\cK \to \bR$ be a convex
function. Suppose $f$ is $1$-strongly convex at $x_0$. Then $f^*$, the
Legendre transform of $f$, is $1$-strongly smooth at
$y_0 = \nabla f(x_0)$.
\end{lemma}

\begin{proof}
  Notice first that for any pair of convex functions
  $f,g : \cK \to \bR$, the fact that
  $f(x_0) \geq g(x_0)$ for some $x_0 \in \cK$ implies that
  $f^*(y_0) \leq g^*(y_0)$ for $y_0 = \nabla f(x_0)$.

  Now, $f$ being $1$-strongly convex at $x_0$ means that
  $f(x) \geq h(x) = f(x_0) + g_0 \cdot (x - x_0) + \frac{\sigma}{2}
  \|x-x_0\|_2^2$.
  Thus, it suffices to show that
  $h^*(y) = f^*(y_0) + x_0 \cdot (y - y_0) + \frac{1}{2} \|y -
  y_0\|_2^2$, since $x_0 = \nabla (h^*)(y_0)$.

To see this, we can compute that
\begin{align*}
 h^*(y) 
&= \max_x y \cdot x - h(x) \\
&= y \cdot (y - y_0 + x_0) - h(x) \\
&\quad \text{ (max attained at $y_0 + (x-x_0) = \nabla h(x) = y $)} \\
&= y \cdot (y-y_0 + x_0) \\
&\quad - \left[ f(x_0) + y_0 \cdot (x-x_0) + \frac{1}{2} \|x-x_0\|_2^2 \right] \\
&=\frac{1}{2} \|y - y_0\|_2^2 + y \cdot x_0 - f(x_0) \\
&= -f(x_0) + x_0 \cdot y_0 + x_0 \cdot (y-y_0) + \frac{1}{2} \|y-y_0\|_2^2 \\
&= f^*(y_0) + x_0 \cdot (y-y_0) + \frac{1}{2} \|y-y_0\|_2^2 
\end{align*}
\end{proof}

\begin{reptheorem}{th:ao_ftrl-gen}[AO-FTRL-Gen]
Let $\{r_t\}$ be a sequence of non-negative functions, and let
$\tilde{g}_t$ be the learner's estimate of $g_t$ given the history of
functions $f_1,\ldots, f_{t-1}$ and points $x_1, \ldots, x_{t-1}$.
Assume further that the function
$h_{0:t}\colon x \mapsto g_{1:t}\cdot x + \tilde{g}_{t + 1}\cdot x +
r_{0:t}(x)$
is 1-strongly convex with respect to some norm $\|\cdot\|_{(t)}$
(i.e. $r_{0:t}$ is 1-strongly convex wrt $\|\cdot\|_{(t)}$).  Then,
the following regret bound holds for AO-FTRL
(Algorithm~\ref{alg:ao_ftrl}):
$$\sum_{t = 1}^T f_t(x_t) - f_t(x) \leq r_{0:T-1}(x) + \sum_{t = 1}^T \|g_t - \tilde{g}_t\|_{(t-1),*}^2$$
\end{reptheorem}

\begin{proof}
Recall that $x_{t + 1} = \argmin_x x \cdot (g_{1:t} + \tilde{g}_{t + 1}) + r_{0:t}(x)$, and let
$y_{t} = \argmin_x x \cdot g_{1:t} + r_{0:t-1}(x).$
Then by convexity,
\begin{align*}
\sum_{t = 1}^T f_t(x_t) - f_t(x) 
&\leq \sum_{t = 1}^T g_t \cdot (x_t - x) \\
&= \sum_{t = 1}^T (g_t - \tilde{g}_t) \cdot (x_t - y_{t})
+ \tilde{g}_t \cdot (x_t - y_{t}) + g_t \cdot (y_{t} - x) 
\end{align*}

Now, we first show via induction that $\forall x \in \cK$, the following holds:
$$\sum_{t = 1}^T \tilde{g}_t \cdot (x_t - y_{t}) + g_t \cdot y_{t}  \leq \sum_{t = 1}^T g_t \cdot x + r_{0:T-1}(x).$$

For $T=1$, the fact that $r_t \geq 0$, $\tilde{g}_1 = 0$, and the definition of $y_{t}$ imply the result.\\

Now suppose the result is true for time $T$. Then
\begin{align*}
\sum_{t = 1}^{T+1} \tilde{g}_t \cdot (x_t - y_{t}) +  g_t \cdot y_{t} 
&\quad = \left[ \sum_{t = 1}^{T} \tilde{g}_t \cdot (x_t - y_{t}) +  g_t \cdot y_{t} \right] \\
&\quad \quad \quad + \tilde{g}_{T+1} \cdot (x_{T+1} - y_{T+1}) + g_{T+1} \cdot y_{T+1} \\
&\quad \leq \left[ \sum_{t = 1}^{T} g_t \cdot x_{T+1} + r_{0:T-1}(x_{T+1}) \right] \\
& \quad \quad \quad + \tilde{g}_{T+1} \cdot (x_{T+1} - y_{T+1}) + g_{T+1} \cdot y_{T+1} \\
&\qquad \text{ (by the induction hypothesis for $x = x_{T+1}$)} \\
&\quad \leq \left[ \left(g_{1:T} + \tilde{g}_{T+1} \right) \cdot x_{T+1} + r_{0:T}(x_{T+1}) \right] \\
& \quad \quad \quad + \tilde{g}_{T+1} \cdot ( - y_{T+1}) + g_{T+1} \cdot y_{T+1} \\
&\qquad \text{ (since $r_t \geq 0$, $\forall t$)} \\
&\quad \leq \left[ \left(g_{1:T} + \tilde{g}_{T+1} \right) \cdot y_{T+1} + r_{0:T}(y_{T+1}) \right] \\
&\quad \quad \quad + \tilde{g}_{T+1} \cdot ( - y_{T+1}) + g_{T+1} \cdot y_{T+1} \\
&\qquad \text{ (by definition of $x_{T+1}$)} \\
&\quad \leq g_{1:T+1} \cdot y + r_{0:T}(y), \text{ for any $y$.}\\
&\qquad \text{ (by definition of $y_{T+1}$)}
\end{align*}
 
Thus, we have that
$\sum_{t = 1}^T f_t(x_t) - f_t(x) \leq r_{0:T-1}(x) + \sum_{t = 1}^T
(g_t - \tilde{g}_t) \cdot (x_t - y_{t}) $
and it suffices to bound $\sum_{t = 1}^T (g_t - \tilde{g}_t)^T(x_t - y_{t})$.
By duality again, one can immediately get $(g_t - \tilde{g}_t) \cdot (x_t - y_{t}) \leq \|g_t - \tilde{g}_t\|_{(t-1),*} \|x_t - y_{t}\|_{(t-1)}$. 
To bound $\|x_t - y_{t}\|_{(t)}$ in terms of the gradient, recall first that 
\begin{align*}
&x_t = \argmin_x h_{0:t-1}(x) \\
&y_{t} = \argmin_x h_{0:t-1}(x) + (g_t - \hat{g}_t) \cdot x. 
\end{align*}
The fact that $r_{0:t-1}(x)$ is 1-strongly convex with respect to the
norm $\|\cdot\|_{(t-1)}$ implies that $h_{0:t-1}$ is as well. In
particular, it is strongly convex at the points $x_t$ and
$y_{t}$. But, this then implies that the conjugate function is smooth
at $\nabla(h_{0:t-1})(x_t)$ and $\nabla(h_{0:t-1} )(y_{t})$, so that
\begin{align*}
&\|\nabla (h_{0:t-1} ^*)(-(g_t - \tilde{g}_t)) \\
&\quad - \nabla (h_{0:t-1}^*)(0)\|_{(t)} \leq \|g_t - \tilde{g}_t\|_{(t-1),*} 
\end{align*}
Since $\nabla (h_{0:t-1}^*)(-(g_t - \tilde{g}_t)) = y_{t}$ and $\nabla (h_{0:t-1} ^*)(0) = x_t$, we have that 
$\|x_t - y_{t}\|_{(t-1)} \leq \|g_t - \tilde{g}_t\|_{(t-1),*}.$

\end{proof}

\begin{reptheorem}{th:cao_ftrl-prox} [CAO-FTRL-Prox]
Let $\{r_t\}$ be a sequence of proximal non-negative functions, such
that $\argmin_{x \in \cK} r_t(x) = x_t$, and let $\tilde{g}_t$ be
the learner's estimate of $g_t$ given the history of functions
$f_1,\ldots, f_{t-1}$ and points $x_1, \ldots, x_{t-1}$.  Let
$\{\psi_t\}_{t = 1}^\infty$ be a sequence of non-negative convex
functions, such that $\psi_1(x_1) = 0$. Assume further that the
function $h_{0:t}: x \mapsto g_{1:t} \cdot x + \tilde{g}_{t+1} \cdot x + r_{0:t}(x) + \psi_{1:t+1}(x)$ 
is 1-strongly convex with respect to some norm $\|\cdot\|_{(t)}$.
Then the following regret bounds hold for CAO-FTRL (Algorithm~\ref{alg:cao_ftrl}):
\begin{align*}
& \sum_{t = 1}^T f_t(x_t) - f_t(x) 
\leq  \psi_{1:T-1}(x) + r_{0:T-1}(x) + \sum_{t = 1}^T \|g_t - \tilde{g}_t\|_{(t-1),*}^2\\
& \sum_{t = 1}^T \left[ f_t(x_t) + \psi_t(x_t) \right] - \left[ f_t(x) + \psi_t(x) \right] 
\leq  r_{0:T}(x) + \sum_{t = 1}^T \|g_t - \tilde{g}_t\|_{(t),*}^2 \, .
\end{align*}
\end{reptheorem}

\begin{proof}
 For the first regret bound, define the auxiliary regularization functions $\tilde{r}_t(x) = r_t(x) + \psi_t(x)$, and apply 
Theorem~\ref{th:ao_ftrl-gen} to get   
\begin{align*}
\sum_{t = 1}^T f_t(x_t) - f_t(x) 
&\leq \tilde{r}_{0:T-1}(x) + \sum_{t = 1}^T \|g_t - \tilde{g}_t\|_{(t-1),*}^2 \\
&= \psi_{1:T-1}(x) + {r}_{0:T-1}(x) + \sum_{t = 1}^T \|g_t - \tilde{g}_t\|_{(t-1),*}^2 
\end{align*}

Notice that while $r_t$ is proximal, $\tilde{r}_t$, in general, is not, and so we must apply the theorem with general regularizers instead of the one
with proximal regularizers. \\

For the second regret bound, we can follow the prescription of Theorem 1 while keeping track of the additional composite terms:

Recall that $x_{t + 1} = \argmin_x x \cdot (g_{1:t} + \tilde{g}_{t + 1}) +
r_{0:t + 1}(x) + \psi_{1:t + 1}(x)$, and let 
$y_{t} = \argmin_x x \cdot g_{1:t} + r_{0:t}(x) + \psi_{1:t}(x)$.\\

We can compute that:
\begin{align*}
\sum_{t = 1}^T f_t(x_t) + \alpha_t \psi(x_t) - \left[ f_t(x) + \psi_t(x) \right]
&\leq \sum_{t = 1}^T g_t \cdot (x_t - x) + \psi_t(x_t) - \psi_t(x) \\
&= \sum_{t = 1}^T (g_t - \tilde{g}_t) \cdot (x_t - y_{t})  \\
&\quad \quad + \tilde{g}_t \cdot (x_t - y_{t}) + g_t \cdot (y_{t} - x) + \psi_t(x_t) - \psi_t(x)
\end{align*}

Similar to before, we show via induction that $\forall x \in \cK$, 
$\sum_{t = 1}^T \tilde{g}_t \cdot (x_t - y_{t}) + g_t \cdot y_{t} + \psi_t(x_t) \leq r_{0:T}(x) + \sum_{t = 1}^T g_t \cdot x + \psi_t(x).$

For $T=1$, the fact that $r_t \geq 0$, $\hat{g}_1 = 0$, $\psi_1(x_1) = 0$,
 and the definition of $y_{t}$ imply the result.\\

Now suppose the result is true for time $T$. Then
\begin{align*}
 \sum_{t = 1}^{T+1} \tilde{g}_t \cdot (x_t - y_{t}) +  g_t \cdot y_{t} + \psi_t(x_t)
&\quad = \left[ \sum_{t = 1}^{T} \tilde{g}_t \cdot (x_t - y_{t}) +  g_t \cdot y_{t} + \psi_t(x_t) \right] \\
&\qquad \qquad + \tilde{g}_{T+1} \cdot (x_{T+1} - y_{T+1}) +  g_{T+1} \cdot y_{T+1} \\
&\qquad \qquad + \psi_{T+1}(x_{T+1}) \\
&\quad \leq \left[ \sum_{t = 1}^{T} g_t \cdot x_{T+1} + r_{0:T}(x_{T+1}) + \psi_t(x_{T+1}) \right] \\
&\qquad \qquad  + \tilde{g}_{T+1} \cdot (x_{T+1} - y_{T+1}) + g_{T+1} \cdot y_{T+1} \\
&\qquad \qquad + \psi_{T+1}(x_{T+1}) \\
&\qquad \qquad \text{ (by the induction hypothesis for $x = x_{T+1}$)} \\
% &= \left[ \left(\sum_{t = 1}^{T} g_t + \hat{g}_{T+1} \right) \cdot x_{T+1} + r_{0:T}(x_{T+1}) \right] + \hat{g}_{T+1} \cdot ( - y_{T+2}) + y_{T+2} \cdot g_{T+1} \\
&\quad \leq \left(g_{1:T} + \tilde{g}_{T+1} \right) \cdot x_{T+1} + r_{0:T+1}(x_{T+1}) + \psi_t(x_{T+1}) \\
&\qquad \qquad + \tilde{g}_{T+1} \cdot ( - y_{T+1}) + g_{T+1} \cdot y_{T+1} \\
&\qquad \qquad + \psi_{T+1}(x_{T+1}) \\
&\qquad \qquad \text{ (since $r_t \geq 0$, $\forall t$)} \\
&\quad \leq \left(g_{1:T} + \tilde{g}_{T+1} \right) \cdot y_{T+1} + r_{0:T+1}(y_{T+1})  + \psi_t(y_{T+1}) \\
&\qquad \qquad + \tilde{g}_{T+1} \cdot ( - y_{T+1}) + g_{T+1} \cdot y_{T+1} \\
&\qquad \qquad + \psi_{T+1}(y_{T+1}) \\
&\qquad \qquad \text{ (by definition of $x_{T+1}$)} \\
% &= \left(\sum_{t = 1}^{T+1} g_t \right) \cdot y_{T+2} + r_{0:T+1}(y_{T+2}) \\
&\quad \leq g_{1:T+1} \cdot y + r_{0:T+1}(y) + \psi_{1:T+1}(y), \text{ for any $y$}\\
&\qquad \qquad \text{ (by definition of $y_{T+1}$)}
\end{align*}

Thus, we have that
\begin{align*}
\sum_{t = 1}^T f_t(x_t) + \psi_t(x_t) - \left[ f_t(x) + \psi_t(x) \right] 
&\leq r_{0:T}(x) + \sum_{t = 1}^T (g_t - \tilde{g}_t)^T (x_t - y_{t}), 
\end{align*}
and we can bound the sum in the same way as before, since the strong convexity properties of $h_{0:t}$ are retained 
due to the convexity of $\psi_t$. 

\end{proof}

\begin{reptheorem}{th:cao_ftrl-gen} [CAO-FTRL-Gen]
Let $\{r_t\}$ be a sequence of non-negative functions, and let
$\tilde{g}_t$ be the learner's estimate of $g_t$ given the history of
functions $f_1,\ldots, f_{t-1}$ and points $x_1, \ldots, x_{t-1}$.
Let $\{\psi_t\}_{t = 1}^\infty$ be a sequence of non-negative convex
functions such that $\psi_1(x_1) = 0$. Assume further that the
function $h_{0:t}: x \mapsto g_{1:t} \cdot x + \tilde{g}_{t+1} \cdot x + r_{0:t}(x) + \psi_{1:t+1}(x)$ 
is 1-strongly convex with respect to some norm $\|\cdot\|_{(t)}$. 
Then, the following regret bound holds for CAO-FTRL (Algorithm~\ref{alg:cao_ftrl}):
\begin{align*}
& \sum_{t = 1}^T f_t(x_t) - f_t(x)
\leq  \psi_{1:T-1}(x) + r_{0:T-1}(x) + \sum_{t = 1}^T \|g_t - \tilde{g}_t\|_{(t-1),*}^2 \\
& \sum_{t = 1}^T f_t(x_t) + \psi_t(x_t) - \left[ f_t(x) + \psi_t(x) \right]
\leq  r_{0:T-1}(x) + \sum_{t = 1}^T \|g_t - \tilde{g}_t\|_{(t),*}^2 \, .
\end{align*}
\end{reptheorem}

\begin{proof}
For the first regret bound, define the auxiliary regularization functions $\tilde{r}_t(x) = r_t(x) + \alpha_t \psi(x)$, and apply Theorem~\ref{th:ao_ftrl-gen} to get   
\begin{align*}
\sum_{t = 1}^T f_t(x_t) - f_t(x) 
&\leq \tilde{r}_{0:T-1}(x) + \sum_{t = 1}^T \|g_t - \hat{g}_t\|_{(t),*}^2 \\
&= \psi_{1:T-1}(x) + {r}_{0:T-1}(x) + \sum_{t = 1}^T \|g_t - \hat{g}_t\|_{(t-1),*}^2 
\end{align*}

For the second bound, we can proceed as in the original proof, but now keep track of the additional composite terms. \\

Recall that $x_{t + 1} = \argmin_x x \cdot (g_{1:t} + \tilde{g}_{t + 1}) + r_{0:t}(x) + \psi_{1:t + 1}(x)$, and let
$y_{t} = \argmin_x x \cdot g_{1:t} + r_{0:t-1}(x) + \psi_{1:t}(x).$
Then
\begin{align*}
\sum_{t = 1}^T f_t(x_t) + \psi_t(x_t) - f_t(x) - \psi_t(x) 
&\leq \sum_{t = 1}^T g_t \cdot (x_t - x) + \psi_t(x_t) - \psi_t(x) \\
&= \sum_{t = 1}^T (g_t - \tilde{g}_t) \cdot (x_t - y_{t})  + \tilde{g}_t \cdot (x_t - y_{t}) \\
&\qquad \qquad + g_t \cdot (y_{t} - x) + \psi_t(x_t) - \psi_t(x)
\end{align*}

Now, we show via induction that $\forall x \in \cK$, 
$\sum_{t = 1}^T \tilde{g}_t \cdot (x_t - y_{t}) + g_t \cdot y_{t} + \alpha_t \psi(x_t) \leq \sum_{t = 1}^T g_t \cdot x + \psi_t(x) + r_{0:T-1}(x).$

For $T=1$, the fact that $r_t \geq 0$, $\hat{g}_1 = 0$, $\psi_1(x_1) = 0$,
 and the definition of $y_{t}$ imply the result.\\

Now suppose the result is true for time $T$. Then
\begin{align*}
 \sum_{t = 1}^{T+1} \tilde{g}_t \cdot (x_t - y_{t}) +  g_t \cdot y_{t} + \psi_t(x_t) 
&\quad = \left[ \sum_{t = 1}^{T} \tilde{g}_t \cdot (x_t - y_{t}) +  g_t \cdot y_{t} + \psi_t(x_t) \right] \\
&\qquad \qquad + \tilde{g}_{T+1} \cdot (x_{T+1} - y_{T+1}) +  g_{T+1} \cdot y_{T+1} \\
&\qquad \qquad + \psi_{T+1}(x_{T+1}) \\
&\quad \leq \left[ \sum_{t = 1}^{T} g_t^T x_{T+1} + r_{0:T-1}(x_{T+1}) + \psi_t(x_{T+1}) \right] \\
&\qquad \qquad + \tilde{g}_{T+1} \cdot (x_{T+1} - y_{T+1}) + g_{T+1} \cdot y_{T+1} \\
&\qquad \qquad + \psi_{T+1}(x_{T+1}) \\
&\qquad \qquad \text{ (by the induction hypothesis for $x = x_{T+1}$)} \\
&\quad \leq \left[ \left(g_{1:T} + \tilde{g}_{T+1} \right) \cdot x_{T+1} + r_{0:T}(x_{T+1}) + \psi_t(x_{T+1}) \right] \\
&\qquad \qquad + \tilde{g}_{T+1} \cdot ( - y_{T+1}) + g_{T+1} \cdot y_{T+1} \\
&\qquad \qquad + \psi_{T+1}(x_{T+1}) \\
&\qquad \qquad \text{ (since $r_t \geq 0$, $\forall t$)} \\
&\quad \leq g_{1:T+1} \cdot y_{T+1} + \tilde{g}_{T+1} \cdot y_{T+1} + r_{0:T}(y_{T+1}) \\
&\qquad \qquad + \psi_{1:T+1}(y_{T+1})  \\
&\qquad \qquad + \tilde{g}_{T+1} \cdot ( - y_{T+1}) + g_{T+1} \cdot y_{T+1} \\
&\qquad \qquad \text{ (by definition of $x_{T+1}$)} \\
&\quad \leq g_{1:T+1} \cdot y + r_{0:T}(y) + \psi_{1:T+1}(y), \text{ for any $y$}\\
&\qquad \text{ (by definition of $y_{T+1}$)}
\end{align*}
 
Thus, we have that $\sum_{t = 1}^T f_t(x_t) + \psi_t(x_t) - f_t(x) - \psi_t(x) \leq r_{0:T-1}(x) + \sum_{t = 1}^T (g_t - \tilde{g}_t) \cdot (x_t - y_{t}) $
and the remainder follows as in the non-composite setting since the strong convexity properties are retained. \\

\end{proof}

\section{Proofs for Section~\ref{sec:aogd}}
\label{app:aogd}

The following lemma is central to the derivation of regret bounds for
many algorithms employing adaptive regularization.  Its proof, via
induction, can be found in Auer et al (2002).
\begin{lemma}
\label{lem:adagrad}
 Let $\{a_j\}_{j=1}^\infty$ be a sequence of non-negative numbers. Then 
$\sum_{j=1}^t \frac{a_j}{\sum_{k=1}^j a_k} \leq 2 \sqrt{\sum_{j=1}^t a_j}.$
\end{lemma}

\begin{repcorollary}{cor:aogd}[AO-GD]
Let $\cK \subset \times_{i = 1}^n [-R_i,R_i]$ be an
$n$-dimensional rectangle, and denote $\Delta_{s,i} = \sqrt{\sum_{a = 1}^s (g_{a, i} - \widetilde g_{a, i})^2}$. Set
\begin{align*}
&r_{0:t} = \sum_{i=1}^n \sum_{s=1}^t \tfrac{\Delta_{s,i} - \Delta_{s-1,i} }{2 R_i} (x_i - x_{s,i})^2. 
\end{align*}
Then, if we use the martingale-type gradient
prediction $\tilde{g}_{t + 1} = g_t$, the following regret bound holds:
$$\text{Reg}_T(x) \leq 4 \sum_{i = 1}^n R_i \sqrt{\sum_{t = 1}^T (g_{t,i} - g_{t-1,i})^2}.$$
Moreover, this regret bound is nearly equal to the optimal a posteriori regret bound:
\begin{align*}
\max_i R_i \sum_{i=1}^n \sqrt{\sum_{t = 1}^T (g_{t,i} - g_{t-1,i})^2} 
&=\max_i R_i \sqrt{ n \inf_{s \succcurlyeq 0, \langle s, 1 \rangle \leq n} \sum_{t=1}^T \|g_t - g_{t-1}\|_{\diag(s)^{-1}}^2} 
\end{align*}
% On the other hand, if we use the empirical average as our gradient prediction $\tilde{g}_{t + 1} = \frac{1}{t} \sum_{s=1}^t g_s$, we get the regret bound:
% $$\text{Reg}_T(x) \leq 4 R \sqrt{\sum_{t = 1}^T \left\| g_t - \frac{1}{t-1}\sum_{s=1}^{t-1} g_s \right\|_2^2}.$$  
\end{repcorollary}

\begin{proof}
  $r_{0:t}$ is $1$-strongly convex with respect to the norm:
$$\|x\|_{(t)}^2 = \sum_{i = 1}^n \frac{\sqrt{ \sum_{a=1}^t (g_{a,i} - \tilde{g}_{a,i})^2}}{R_i} x_i^2,$$
which has corresponding dual norm:
$$\|x\|_{(t),*}^2 = \sum_{i = 1}^n \frac{R_i}{\sqrt{ \sum_{a=1}^t (g_{a,i} - \tilde{g}_{a,i})^2}} x_i^2.$$
By the choice of this regularization, the prediction
$\tilde{g}_t = g_{t-1}$, and
Theorem~\ref{th:cao_ftrl-prox}, the following holds:
\begin{align*}
\Reg_T(\cA, x) 
&\leq \sum_{i = 1}^n \sum_{s=1}^T \frac{\sqrt{\sum_{a=1}^s (g_{a,i} - \tilde{g}_{a,i})^2} - \sqrt{\sum_{a=1}^{s-1} (g_{a,i} - \tilde{g}_{a,i})^2} }{2 R_i} (x_i - x_{s,i})^2 \\
&\quad + \sum_{t = 1}^T \|g_t - g_{t-1}\|_{(t),*}^2 \\
&= \sum_{i = 1}^n 2 R_i \sqrt{\sum_{t = 1}^T (g_{t,i} - g_{t-1,i})^2} \\
&\quad + \sum_{i = 1}^n \sum_{t = 1}^T \frac{R_i (g_{t,i} - g_{t-1,i})^2 }{\sqrt{ \sum_{a=1}^t (g_{a,i} - g_{a-1,i})^2}}  \\
&\leq \sum_{i = 1}^n 2 R_i \sqrt{\sum_{t = 1}^T (g_{t,i} - g_{t-1,i})^2} \\
&\quad + \sum_{i = 1}^n 2 R_i \sqrt{ \sum_{t = 1}^T (g_{t,i} - g_{t-1,i})^2} \\
&\quad \text{ by Lemma~\ref{lem:adagrad}} \\
\end{align*}
The last statement follows from the fact that
$$\inf_{s\succcurlyeq 0, \langle s, 1, \rangle \leq n} \sum_{t=1}^T \sum_{i=1}^n \frac{g_{t,i}^2}{s_i} = \frac{1}{n} \left(\sum_{i=1}^n \|g_{1:T}, i\|_2 \right)^2,$$
since the infimum on the left hand side is attained when $s_i \propto \|g_{1:T,i}\|_2.$ 
\end{proof}

\section{Proofs for Section~\ref{sec:aoeg}}
\label{app:aoeg}

\begin{repcorollary}{cor:aoeg}[AO-EG]
Let $\cK = \Delta_n$ be the $n$-dimensional simplex and
$\phi\colon x \mapsto \sum_{i = 1}^n x_i \log(x_i)$ the negative
entropy. Assume that $\|g_t\| \leq C$ for all $t$ and set
$$r_{0:t} = \sqrt{2\frac{C + \sum_{s=1}^t \|g_s -
    \tilde{g}_s\|_\infty^2}{\log(n)}} (\phi + \log(n)).$$
Then, if we use the martingale-type gradient prediction
$\tilde{g}_{t + 1} = g_t$ the following regret bound holds:
$$\Reg_T(\cA, x) \leq 2 \sqrt{2 \log(n) \left( C + \sum_{t = 1}^{T-1} \|g_t - g_{t-1}\|_\infty^2 \right) }.$$
% On the other hand, if we use the empirical average as our gradient prediction $\tilde{g}_{t + 1} = \frac{1}{t} \sum_{s=1}^t g_s$, we get the regret bound:
% $$\Reg_T(\cA, x) \leq 4 \sqrt{\log(n) \left( C + \sum_{t = 1}^{T-1} \left\| g_t - \frac{1}{t-1}\sum_{s=1}^{t-1} g_s \right\|_2^2 \right) }.$$   
\end{repcorollary}

\begin{proof}
 Since the negative entropy $\phi$ is $1$-strongly convex with respect to the $l_1$-norm, $r_{0:t}$ is 
$\sqrt{2 \frac{C + \sum_{s=1}^t \|g_s - \tilde{g}_s\|_\infty^2}{\log(n)}}$-strongly convex with respect to the same norm. \\

Applying Theorem~\ref{th:ao_ftrl-gen} and using the fact that the dual of $l_1$ is $l_\infty$ along with $\varphi \leq 0$ yields a regret bound of:
\begin{align*}
\Reg_T(\cA, x) 
&\quad \leq r_{0:T-1}(x) + \sum_{t = 1}^T \|g_t - \tilde{g}_t\|_{(t-1),*}^2 \\
&\quad \leq \sqrt{2 \frac{C + \sum_{s = 1}^{T-1} \|g_s - \tilde{g}_s\|_\infty^2}{\log(n)}} (\phi + \log(n)) \\
&\qquad + \sum_{t = 1}^T \frac{1}{\sqrt{2}} \sqrt{\frac{\log(n)}{C + \sum_{s=1}^{t-1} \|g_s - \tilde{g}_s\|_\infty^2}} \|g_t - \tilde{g}_t\|_\infty^2 \\
&\quad \leq \sqrt{ 2 \left(C + \sum_{s=1}^{T-1} \|g_s - \tilde{g}_s\|_\infty^2 \right) \log(n)} \\
&\qquad + \sum_{t = 1}^T \frac{1}{\sqrt{2}} \sqrt{\frac{\log(n)}{\sum_{s=1}^t \|g_s - \tilde{g}_s\|_\infty^2}} \|g_t - \tilde{g}_t\|_\infty^2 \\
&\quad \leq \sqrt{ 2 \left(C + \sum_{s=1}^{T-1} \|g_s - \tilde{g}_s\|_\infty^2 \right) \log(n)}  \\
&\qquad + \sqrt{2 \log(n) \sum_{t = 1}^T \|g_t - \tilde{g}_t\|_\infty^2} \\
&\quad \leq 2 \sqrt{ 2 \left(C + \sum_{s=1}^{T-1} \|g_s - \tilde{g}_s\|_\infty^2 \right) \log(n)}.  
\end{align*}
\end{proof}

\section{Proofs for Section~\ref{sec:aos_ftrl}}
\label{app:aos_ftrl}

\begin{reptheorem}{th:caos_ftrl-prox}[CAOS-FTRL-Prox]
Let $\{r_t\}$ be a sequence of proximal non-negative functions, such
that $\argmin_{x \in \cK} r_t(x) = x_t$, and let $\tilde{g}_t$ be the
learner's estimate of $\hat{g}_t$ given the history of noisy gradients
$\hat{g}_1,\ldots, \hat{g}_{t-1}$ and points $x_1, \ldots, x_{t-1}$.
Let $\{\psi_t\}_{t = 1}^\infty$ be a sequence of non-negative convex
functions, such that $\psi_1(x_1) = 0$. Assume further that the
function
$$h_{0:t}(x) = \hat{g}_{1:t}\cdot x + \tilde{g}_{t + 1}\cdot x + r_{0:t}(x) + \psi_{1:t + 1}(x)$$
is 1-strongly convex with respect to some norm $\|\cdot\|_{(t)}$.
Then, the update $x_{t + 1} = \argmin_x h_{0:t}(x)$ of
Algorithm~\ref{alg:caos_ftrl} yields the following regret bounds:
\begin{align*}
&\bE\left[ \sum_{t = 1}^T f_t(x_t) - f_t(x) \right] 
\leq \bE \left[ \psi_{1:T-1}(x) + r_{0:T-1}(x) + \sum_{t = 1}^T \|\hat{g}_t - \tilde{g}_t\|_{(t-1),*}^2 \right] \\
&\bE\left[ \sum_{t = 1}^T f_t(x_t) + \psi_t(x_t) - f_t(x) - \alpha_t \psi_t(x) \right] 
\leq \bE \left[ r_{0:T}(x) + \sum_{t = 1}^T \|\hat{g}_t - \tilde{g}_t\|_{(t),*}^2 \right]. 
\end{align*}
\end{reptheorem}

\begin{proof}

\begin{align*}
 \bE \left[ \sum_{t = 1}^T f_t(x_t) - f_t(x) \right] 
&\quad \leq \sum_{t = 1}^T \bE \left[ g_t \cdot (x_t - x) \right] \\
% &= \sum_{t = 1}^T \bE \left[ \bE[ g_t^T(x_t - x) | g_1,\ldots, g_{t-1}, x_1, \ldots, x_t ]\right] \\
% &= \sum_{t = 1}^T \bE \left[ \bE[ g_t | g_1,\ldots, g_{t-1}, x_1, \ldots, x_t ]^T (x_t - x)\right] \\
&\quad = \sum_{t = 1}^T \bE \left[ \bE[ \hat{g}_t | \hat{g}_1,\ldots, \hat{g}_{t-1}, x_1, \ldots, x_t ]^T (x_t - x)\right] \\
&\quad = \sum_{t = 1}^T \bE \left[ \bE[ \hat{g}_t \cdot (x_t - x) | \hat{g}_1,\ldots, \hat{g}_{t-1}, x_1, \ldots, x_t ] \right] \\
&\quad = \sum_{t = 1}^T \bE \left[ \hat{g}_t \cdot (x_t - x) \right] \\
% &= \bE \left[ \sum_{t = 1}^T \hat{g}_t^T (x_t - x) \right] 
% &\leq \bE \left[ r_{0:T}(x^*) + \sum_{t = 1}^T \|\hat{g}_t - \tilde{g}_t \|_{(t),*}^2 \right] \\
% &\quad \text{ (by running AO-FTRL-Prox on the gradient estimates)} 
\end{align*}
This implies that upon taking an expectation, we can freely upper bound the difference $f_t(x_t) - f_t(x)$ by the noisy linearized estimate $\hat{g}_t \cdot(x_t-x)$. 
After that, we can apply Algorithm~\ref{alg:cao_ftrl} on the gradient estimates to get the bounds: 
\begin{align*}
&\bE\left[ \sum_{t = 1}^T \hat{g}_t^T(x_t-x) \right] 
\leq \bE \left[ \psi_{1:T-1}(x) + r_{0:T-1}(x) + \sum_{t = 1}^T \|\hat{g}_t - \tilde{g}_t\|_{(t-1),*}^2 \right] \\
&\bE\left[ \sum_{t = 1}^T \hat{g}_t^T(x_t-x)+ \psi_t(x_t) - \psi_t(x) \right] 
\leq \bE \left[ r_{0:T}(x) + \sum_{t = 1}^T \|\hat{g}_t - \tilde{g}_t\|_{(t),*}^2 \right] 
\end{align*}
 
\end{proof}

\begin{reptheorem}{th:caos_ftrl-gen}[CAOS-FTRL-Gen]
Let $\{r_t\}$ be a sequence of non-negative functions, and let
$\tilde{g}_t$ be the learner's estimate of $\hat{g}_t$ given the
history of noisy gradients $\hat{g}_1,\ldots, \hat{g}_{t-1}$ and
points $x_1, \ldots, x_{t-1}$.  Let $\{\psi_t\}_{t = 1}^\infty$ be a
sequence of non-negative convex functions, such that
$\psi_1(x_1) = 0$. Assume furthermore that the function
$$h_{0:t}(x) = \hat{g}_{1:t}\cdot x + \tilde{g}_{t + 1}\cdot x + r_{0:t}(x) + \psi_{1:t + 1}(x)$$
is 1-strongly convex with respect to some norm $\|\cdot\|_{(t)}$. Then, the update $x_{t + 1} = \argmin_x h_{0:t}(x)$ of
Algorithm~\ref{alg:caos_ftrl} yields the regret bounds:
\begin{align*}
&\bE\left[ \sum_{t = 1}^T f_t(x_t) - f_t(x) \right]
\leq \bE \left[ \psi_{1:T-1}(x) + r_{0:T-1}(x) + \sum_{t = 1}^T \|\hat{g}_t - \tilde{g}_t\|_{(t-1),*}^2 \right] \\
&\bE\left[ \sum_{t = 1}^T f_t(x_t) + \psi_t(x_t) - f_t(x) - \psi_t(x) \right]
\leq \bE \left[ r_{0:T-1}(x) + \sum_{t = 1}^T \|\hat{g}_t - \tilde{g}_t\|_{(t-1),*}^2 \right] 
\end{align*}
\end{reptheorem}

\begin{proof}
 The argument is the same as for Theorem~\ref{th:caos_ftrl-prox}, except that we now apply the bound of Theorem~\ref{th:cao_ftrl-gen} at the end.
\end{proof}

\section{Proofs for Section~\ref{sec:rcd}}
\label{app:rcd}

\begin{reptheorem}{th:cao_rcd}[CAO-RCD]
Assume $\cK \subset \times_{i = 1}^n [-R_i,R_i]$. Let $i_t$ be a random variable sampled according to the distribution $p_t$, and let
$$\hat{g}_t = \frac{(g_t \cdot e_{i_t}) e_{i_t}}{p_{t,i_t}},\quad \hat{\tilde{g}}_t = \frac{(\tilde{g}_t \cdot e_{i_t}) e_{i_t}}{p_{t,i_t}},$$
be the estimated gradient and estimated gradient prediction. Denote $\Delta_{s,i} = \sqrt{\sum_{a=1}^s (\hat{g}_{a,i} -\hat{\tilde{g}}_{a,i})^2}$, and let
\begin{align*}
r_{0:t} = \text{ }\sum_{i = 1}^n \sum_{s=1}^t \frac{\Delta_{s,i} - \Delta_{s-1,i} }{2 R_i} (x_i-x_{s,i})^2 
\end{align*}
be the adaptive regularization. Then the regret of the resulting algorithm is bounded by:
\begin{align*}
\bE \left[ \sum_{t = 1}^T f_t(x_t) + \alpha_t \psi(x_t) - f_t(x) - \alpha_t \psi(x) \right]
&\leq 4 \sum_{i = 1}^n R_i \sqrt{ \sum_{t = 1}^T \bE \left[ \frac{ (g_{t,i} - \tilde{g}_{t,i})^2}{p_{t,i}} \right] }. 
\end{align*}
\end{reptheorem}

\begin{proof}
We can first compute that
\begin{align*}
\bE \left[\hat{g}_t \right]  
= \bE \left[ \frac{(g_t \cdot e_{i_t}) e_{i_t}}{p_{t,i_t}} \right] 
= \sum_{i = 1}^n \frac{(g_t \cdot e_{i}) e_{i}}{p_{t,i}} p_{t,i} 
= g_t
\end{align*}
and similarly for the gradient prediction $\tilde{g}_t$. \\

Now, as in Corollary~\ref{cor:aogd}, the choice of regularization ensures us a regret bound of the form: 
\begin{align*}
\bE \left[ \sum_{t = 1}^T f_t(x_t) + \alpha_t \psi(x_t) - f_t(x) - \alpha_t \psi(x) \right] 
&\leq 4 \sum_{i = 1}^n R_i \bE \left[ \sqrt{\sum_{t = 1}^T (\hat{g}_{t,i} - \tilde{g}_{t,i})^2} \right] 
\end{align*}

Moreover, we can compute that: 
\begin{align*}
 \bE \left[ \sqrt{ \sum_{t = 1}^T (\hat{g}_{t,i} - \tilde{g}_{t,i})^2} \right] 
&\leq \sqrt{ \bE \left[ \sum_{t = 1}^T \bE_{i_t} [ (\hat{g}_{t,i} - \tilde{g}_{t,i})^2 ] \right] } \\
% &= \sqrt{ \sum_{t = 1}^T \bE \left[ \left\| \frac{(g_t \cdot e_{i_t}) e_{i_t}}{p_{t,i_t}} - \frac{(\tilde{g}_t \cdot e_{i_t}) e_{i_t}}{p_{t,i_t}} \right\|_2^2 \right] }\\
% &= \sqrt{ \sum_{t = 1}^T \bE \left[ \left\| \frac{( (g_t - \tilde{g}_t)\cdot e_{i_t}) e_{i_t}}{p_{t,i_t}} \right\|_2^2 \right] } \\
% &= \sqrt{ \sum_{t = 1}^T \bE \left[ \frac{ (g_{t,i_t} - \tilde{g}_{t,i_t})^2}{p_{t,i_t}^2} \right] }\\
% &= \sqrt{ \sum_{t = 1}^T \sum_{j=1}^n \frac{ (g_{t,j} - \tilde{g}_{t,j})^2}{p_{t,j}^2} p_{t,j} } \\
&= \sqrt{ \sum_{t = 1}^T \bE \left[ \frac{ (g_{t,i} - \tilde{g}_{t,i})^2}{p_{t,i}} \right] }\\
% &= \sqrt{ \sum_{t = 1}^T \sum_{j=1}^n \frac{ (g_{t,j} - \tilde{g}_{t,j})^2 \sum_{i = 1}^n L_{t,i}^2}{L_{t,j}^2} }\\
\end{align*}
 
\end{proof}

% \begin{corollary}[CAO-RCD-Lipschitz-Mini-Batch]
% \label{cor:cao_rcd_lip_batch}
% Assume $\cK \subset \times_{i = 1}^n [-R_i,R_i]$. Let  $\cup_{j=1}^k \{\Pi_j\} = \{1,\ldots, n\}$ be a partition of the coordinates, and let 
% $e_{\Pi_j} = \sum_{i \in \Pi_j} e_i$. 
% Assume we had the following Lipschitz condition on the partition: $\|g_t \cdot e_{\Pi_j} \| \leq L_{j}$ $\forall j \in \{1,\ldots, k\}$.  
% 
% Define $S_i =  \sum_{j \in \Pi_i} R_j $. Set $p_{t,i} = \frac{(S_i L_{i})^{2/3}}{\sum_{j=1}^k (S_j L_{j})^{2/3}}$ as the 
% probability distribution at time $t$, and set $\tilde{g}_{t,i} = \frac{L_{i}}{2}$.
% 
% Then the regret of the resulting algorithm is bounded by:
% \begin{align*}
% \bE \left[ \sum_{t = 1}^T f_t(x_t) + \alpha_t \psi(x_t) - f_t(x) - \alpha_t \psi(x) \right]
% &\leq 4 \sqrt{T} \left( \sum_{i = 1}^k (S_i L_i)^{2/3} \right)^{3/2}    
% \end{align*}
% 
%   
% \end{corollary}

\section{Further Discussion for Section~\ref{sec:reg_erm}}
\label{app:reg_erm}

We present here Algorithm~\ref{alg:caos_reg-erm-epoch-batch}, a mini-batch version of Algorithm~\ref{alg:caos_reg-erm-epoch}, with an accompanying guarantee.  

\begin{algorithm}[t]
\caption{CAOS-Reg-ERM-Epoch-Mini-Batch}
\label{alg:caos_reg-erm-epoch-batch}
\begin{algorithmic}[1]
\State \textbf{Input:} scaling constant $\alpha > 0$, composite term $\psi$, $r_0 = 0$, partitions $\cup_{j=1}^l \{\Pi_j\} = \{1,\ldots, m\}$. 
\State \textbf{Initialize: } initial point $x_1 \in \cK$, distribution $p_1$ over $\{1,\ldots, l\}$. 
\State Sample $j_1$ according to $p_{1}$, and set $t=1$. 
\For{$s = 1, \ldots, k$:}
    \State Compute $\bar{g}_s^j = \nabla f_j(x_1)$ $\forall j \in \{1,\ldots, m\}$.
    \For{$a = 1, \ldots, T/k$:} 
      \State If $T \mod k = 0$, compute $g^j = \nabla f_j(x_t)$ $\forall j$. 
      \State Set $\hat{g}_t = \frac{\sum_{j \in \Pi_{j_t}} g_t^{j}}{p_{t,j_t}}$, and construct $r_t \geq 0$.
      \State Sample $j_{t+1} \sim p_{t+1}$.
      \State Set $\tilde{g}_{t + 1} = \frac{\sum_{j \in \Pi_{j_t}} \bar{g}_s^{j}}{p_{t,j_t}}$. 
      \State Update $x_{t + 1} = \argmin_{x \in \cK} \hat{g}_{1:t} \cdot x + \tilde{g}_{t + 1} \cdot x + r_{0:t}(x) + (t+1)\alpha\psi(x)$ and $t = t+1$. 
    \EndFor
\EndFor
\end{algorithmic}
\end{algorithm}
 
\begin{corollary}
\label{cor:aos_erm-epoch-batch}
Assume $\cK \subset \times_{i = 1}^n [-R_i,R_i]$. 
Let $\cup_{j=1}^l \{\Pi_j\} = \{1,\ldots, n\}$ be a partition of the functions $f_i$, and let $e_{\Pi_j} = \sum_{i \in \Pi_j} e_i$.
Denote $\Delta_{s,i} = \sqrt{\sum_{a=1}^s(\hat{g}_{a,i} - \tilde{g}_{a,i})^2}$, and let
$r_{0:t} = \sum_{i = 1}^n \sum_{s=1}^t \frac{\Delta_{s,i} - \Delta_{s-1,i} }{2 R_i}(x_i-x_{s,i})^2$
be the adaptive regularization. 

Then the regret of Algorithm~\ref{alg:caos_reg-erm-epoch-batch} is bounded by:
\begin{align*}
\bE \left[ \sum_{t = 1}^T f_t(x_t) + \alpha \psi(x_t) - f_t(x) - \alpha \psi(x) \right]
&\leq \sum_{i=1}^n 4 R_i \sqrt{ \sum_{s=1}^k \sum_{t=(s-1)(T/k) + 1}^{(s-1)(T/k)+T/k} \sum_{a=1}^l \frac{\left|\sum_{j \in \Pi_j} g_{t,i}^j - \bar{g}_{s,i}^j \right|^2}{p_{t,a} } }  
% 4 \sqrt{ \sum_{t = 1}^T \bE \left[ \left\| \hat{g}_t - \hat{g}_{t-1} \right\|^2 \right] } 
\end{align*}

Moreover, if $\|\nabla f_j\|_{\infty} \leq L_j$ $\forall j$, then setting $p_{t,j} = \frac{L_i}{\sum_{j=1}^m L_j}$ yields a worst-case bound of:
$8\sum_{i=1}^n R_i \sqrt{T \left(\sum_{j=1}^m L_j\right)^2}.$ 
\end{corollary}

A similar approach to Regularized ERM was developed independently by
\citep{ZhaoZhang2014}. However, the one here improves upon that
algorithm through the incorporation of adaptive regularization,
optimistic gradient predictions, and the fact that we do not assume
higher regularity conditions such as strong convexity for our loss functions.

\end{document}